\documentclass[]{article}
\usepackage{indentfirst}
\usepackage{amsfonts}
\usepackage{amsmath}
\usepackage{textcomp}
\usepackage{graphicx}
\graphicspath{ {./Figures/} }
\usepackage{caption}
\usepackage{subcaption}
\usepackage{float}
\usepackage{amsthm}
\usepackage{amssymb}
\usepackage{gensymb}
\DeclareMathOperator{\tr}{tr}
\DeclareMathOperator{\KL}{KL}
\usepackage{tikz}
\usepackage{pgfplots}
\usetikzlibrary{arrows}

\pgfmathdeclarefunction{gauss}{2}{%
	\pgfmathparse{1/(#2*sqrt(2*pi))*exp(-((x-#1)^2)/(2*#2^2))}%
}

\title{Comparison of Generative Adversarial Networks Architectures Which Reduce Mode Collapse}
\author{Yicheng (Katherine) Hong}

\begin{document}
	\maketitle
	\begin{abstract}
		\normalsize
		Generative Adversarial Networks are known for their high quality outputs and versatility. However, they also suffer the mode collapse in their output data distribution. There have been many efforts to revamp GANs model and reduce mode collapse. This paper focuses on two of these models, PacGAN and VEEGAN. This paper explains the mathematical theory behind aforementioned models, and compare their degree of mode collapse with vanilla GAN using MNIST digits as input data. The result indicates that PacGAN performs slightly better than vanilla GAN in terms of mode collapse, and VEEGAN performs worse than both PacGAN and vanilla GAN. VEEGAN's poor performance may be attributed to average autoencoder loss in its objective function and small penalty for blurry features. 
	\end{abstract}

	\section{Introduction}
	
	A Generative Adversarial Network (GAN) is a neural network consisting of a generator and a discriminator.The network is trained with a set of data. The generator takes a random noise vector as input and generates data that resembles as much as the provided data possible. The discriminator reads a piece of data, either generated or from the actual dataset, and labels the data as real or fake.  For example, in this paper, I trained GANs with MNIST digits, a dataset of handwritten digits. In the network, the generator takes a vector whose elements are generated randomly from a normal distribution, and outputs an image. The input of the discriminator is an image either from the generator output or from the MNIST dataset. The discriminator identifies the input image as real (labeled as 1) or fake (labeled as 0).The degree of resemblance of generated data to data from the actual dataset is measured by cross entropy between a vector of ones (the ideal situation where all generated data is labeled real) and the labels of generated images. In the early stage of training, the images generated by the generator are mostly labeled as fake by the discriminator. After many epochs, the generator begins to output images that resemble handwritten digits, and these images are more likely to be recognized by the discriminator as real images from the MNIST digit dataset. \\
	
	GAN has the advantage of generating better samples than other generative models. Also, a GAN model can train any kind of generator, unlike other models which require the generator to have a particular functional form. However, GAN also has several limitations. First, it is difficult and slow to train. Using my own experiments as an example, while training DCGAN (Deep Convolutional Generative Adversarial Networks) with MNIST digits, it took 50 epochs for the generator to produce images that resemble MNIST digits. After 100 epochs, some of the images produced by the generator still possess features that resembles multiple digits or none of the digits. Another limitation, which is be the primary topic addressed in this paper, is mode collapse, where the occurrence of a certain type of data in the generated distribution is much fewer than its occurrence in the distribution of the provided dataset. I will use my DCGAN experiments with MNIST digits again as an example. The MNIST training dataset has 60,000 images of handwritten digits, with 6000 images of each digit. The training dataset provided to DCGAN is uniform. In every trial of my experiment, DCGAN generates 100 images. Among the 1000 images in the ten trials I conducted, digit 2 is only generated 22 times. However, if DCGAN generated data in proportion of MNIST training dataset, there would be approximately 100 images of digit 2 among the 1000 images. Such discrepancy between the distributions of generated and actual data is an example of mode collapse.\\
	
	Several improvements to the original GAN architecture \cite[Section~6]{pacgan} have been proposed to mitigate mode collapse. Among them, there are VEEGAN (Variational Encoder Enhancement to Generative Adversarial Networks), unrolled GAN, dropout GAN, and PacGAN. In this paper, I will discuss the theory of PacGAN \cite{pacgan} and VEEGAN \cite{veegan}, two improved architectures aimed to reduce mode collapse. I will also present experimental results of mode collapse reduction using the MNIST digits with the two mentioned architecture and compare the results with MNIST digits generated by DCGAN. The MNIST digits generated by DCGAN is the control group in mode collapse reduction experiment.\\
 	
 	The intuitive reason why PacGAN is effective in reducing mode collapse is because the probability distribution of m data samples packed together is an m\textsuperscript{th} product distribution, and the difference in the mode collapse region is more pronounced in the m\textsuperscript{th} degree than the first degree. The intuition behind VEEGAN is that it adds a reconstructor into the generator/discriminator network. The reconstructor takes actual data or data created by the generator as input and outputs noise under a Gaussian distribution. The reconstructor has two objectives: to act as the inverse of the generator and to map all data to Gaussian noise. If both objectives are achieved, the reconstructor in turn helps the generator map Gaussian noise to synthetic data matching the probability distribution of actual data.
	
	\section{PacGAN} 
	
	PacGAN, a modification of GAN proposed by Zinan Lin et al., inputs m independently and randomly selected data samples into the discriminator at the same time \cite{pacgan}. The data samples are either all selected from the actual dataset or all from the generator output. If a single data sample is selected from a dataset with probability distribution $T$, then m independently and randomly selected samples will have a probability distribution of $T^{m}$. The authors of PacGAN framework claimed that mode collapse is more pronounced in $T^{m}$ than $T$, and will be easier to avoid in $T^{m}$ \cite{pacgan}. Before I present the proof of their claim, I need to first introduce both their algebraic and geometric definitions of mode collapse and how to measure mode collapse \cite{pacgan}. In the rest of Section 2, $P$ is the probability distribution of the data provided to GAN for training, also called the target distribution. $Q$ is the probability distribution of the data produced by the generator, also called the generated distribution. The mode collapse between target distribution $P$ and generated distribution $Q$ is defined in Definition \ref{modecollapsedef}.\\
	
	\newtheorem{dfn}{Definition}
	\begin{dfn}[This is Definition 1 in \cite{pacgan}]
	\label{modecollapsedef}
	Between the target distribution $P$ and the generated distribution $Q$  with a common domain D, we say $P$ and $Q$ have $(\varepsilon,\delta)  $ mode collapse if for $ 0\leq\varepsilon<\delta\leq 1 $, there exist a set $ S \subseteq D $ such that $ P(S)\geq\delta $ and $ Q(S)\leq\varepsilon $.
	\end{dfn}
	
	We also need to define how mode collapse is measured. One measurement is the area of mode collapse region. In a 2D plane, let $ \varepsilon $ be the x-axis and $ \delta $ eb the y-axis. For a pair of target and generated distributions $(P,Q)$, mode collapse region $R(P,Q)$ is defined as the convex hull of the region of points $ (\varepsilon, \delta) $ such that $(P,Q)$ exhibit $ (\varepsilon, \delta) $ mode collapse. Definition \ref{R(P,Q)def} provides the definition of mode collapse region in mathematical terms as well. An example of mode collapse region  $ R(P,Q) $ is shown in Figure \ref{dTV}, where the blue shaded area denotes $ R(P,Q).$ \\
	
	\begin{dfn}[This is defined in p.17 of \cite{pacgan}]
		\label{R(P,Q)def}
		For a pair of target and generated distributions $(P,Q)$, mode collapse region $R(P,Q)$ is defined as the convex hull of the region of points $ (\varepsilon, \delta) $ such that $(P,Q)$ exhibit $ (\varepsilon, \delta) $ mode collapse, i.e. $ R(P,Q) = \mathrm{conv}({(\varepsilon, \delta) \: | \: \delta > \varepsilon \text{ and } (P,Q) \text{ has } (\varepsilon, \delta) \text{ mode collapse}}) $ where $\mathrm{conv}$(.) denotes the convex hull.
	\end{dfn}

	Mode collapse region is a clear and direct way to indicate the severity of mode collapse. However, if $ P $ and $ Q $ are multivariate probability distributions, mode collapse region can evolve into a volume of high dimensions, and the volume may be intractable. Therefore, we would like to measure mode collapse using a numerical value instead of the area or volume of a geometric object. The authors of PacGAN framework used total variation distance between $ P $ and $ Q $ to measure the mode collapse between $ P $ and $ Q $ \cite{pacgan}.\\ The geometric interpretation of total variation distance between $ P $ and $ Q $ is shown in Figure \ref{dTV}, denoted as $ d_{TV}(P,Q) $ in red. 
		
	\begin{dfn}[This is defined in p.15 of \cite{pacgan}]
	The total variation distance between $P$ and $Q$ is $ d_{TV}(P,Q)= \displaystyle\sup_{S \subseteq D} \, \{P(S)-Q(S)\} $.\\
	\end{dfn}

	\begin{figure}[H]
		\centering
		\includegraphics[scale = 0.5]{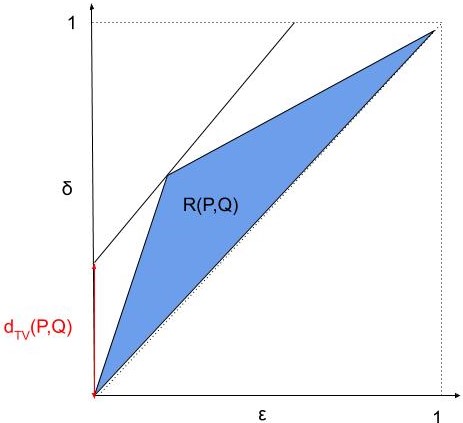}
		\caption{The mode collapse region $ R(P,Q) $ is shaded in blue. Total variation distance between $ P $ and $ Q $, $ d_{TV}(P,Q) $, is denoted as the red segment on the vertical axis.}
		\label{dTV}
	\end{figure}

	As shown in Figure \ref{dTV}, geometrically, this translates to the vertical distance between the upper boundary of the mode collapse region (denoted as $ R(P,Q) $) and $ \varepsilon=\delta $. In other words, the total variation distance between P and Q is the intersection between the vertical axis and the tangent line to the upper boundary of $ R(P,Q) $ that has a slope of 1. Figure \ref{dTV_explained} further explains why this geometric interpretation is equivalent to $ \displaystyle\sup_{S \subseteq D} \, \{P(S)-Q(S)\} $. \\
	
	\begin{figure}[H]
		\centering
		\includegraphics[scale=0.5]{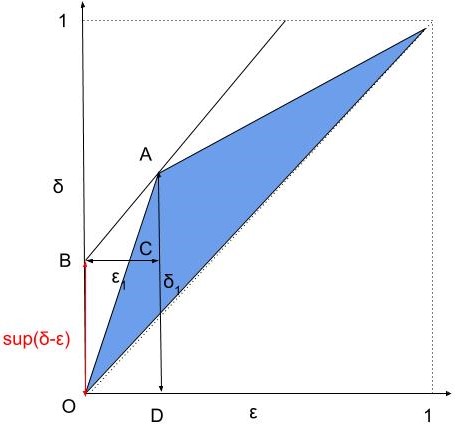}
		\caption{The intersection between the vertical axis and the tangent line to the upper boundary of $ R(P,Q) $ that has a slope of 1 is the maximum difference between $ \varepsilon $ and $ \delta $.}
		\label{dTV_explained}
	\end{figure}

	In Figure \ref{dTV_explained}, line $ AB $ is parallel to line $ \varepsilon = \delta $. $ BC $ is parallel to the x-axis. AC is perpendicular to the x-axis. Because line $ AB $ has a slope of 1 and BC is parallel to the x-axis, $ \angle ABC = 45 \degree $. $ \triangle ABC $ is a right isosceles triangle. Point $A$ is where $ \varepsilon -\delta $ is maximum because $A$ is farthest from the $ \varepsilon = \delta $ line in $ R(P,Q) $. Suppose the coordinates of $ A $ is $ (\varepsilon_{1}, \delta_{1}) $, then $ \varepsilon_{1}-\delta_{1} = \displaystyle\sup_{S \subseteq D} \, \{P(S)-Q(S)\} $. Since $ OB = AD - BC = \varepsilon_{1}-\delta_{1} $, and $ OB = d_{TV}(P,Q) $, we can say that $ d_{TV}(P,Q) $ is the intersection between the vertical axis and the tangent line to the upper boundary of $ R(P,Q) $ that has a slope of 1.\\
	
	The authors of PacGAN framework pointed out that probability distributions with the same total variation distance can manifest drastically different mode collapse behavior. In \cite{pacgan}, they used a simple example to demonstrate this discrepancy. Figure \ref{toyexample} provides an illustration for this example. Let $ U[a, b] $ denote a uniform distribution between a and b with probability density function 
	\[f(x) = \begin{cases}
	\frac{1}{b-a} \text{ for } x \in [a,b] \\
	0 \text{ for } x \notin [a,b] 
	\end{cases}.\]  
	Suppose the target probability distribution is $ P = U[0,1],$ and there are two different generated probability distributions $ Q_{1} $ and $ Q_{2} ,$ with $ Q_{1} = U[0.2, 1] $ and $ Q_{2} = 0.3U[0,0.5] + 0.7U[0.5,1].$ In comparison to target distribution $ P $, $ Q_{1} $ displays more severe mode collapse than $ Q_{2}$, as $ Q_{2}$ addresses the entire domain of $ [0,1] $, while $ Q_{1} $ does not address $ [0,0.2] $ at all. However, their total variation distance to $ P $ are equal: $ d_{TV}(P, Q_{1}) = d_{TV}(P, Q_{2}) = 0.2.$
	
	\begin{figure}[H]
		\centering
		\includegraphics[scale=1]{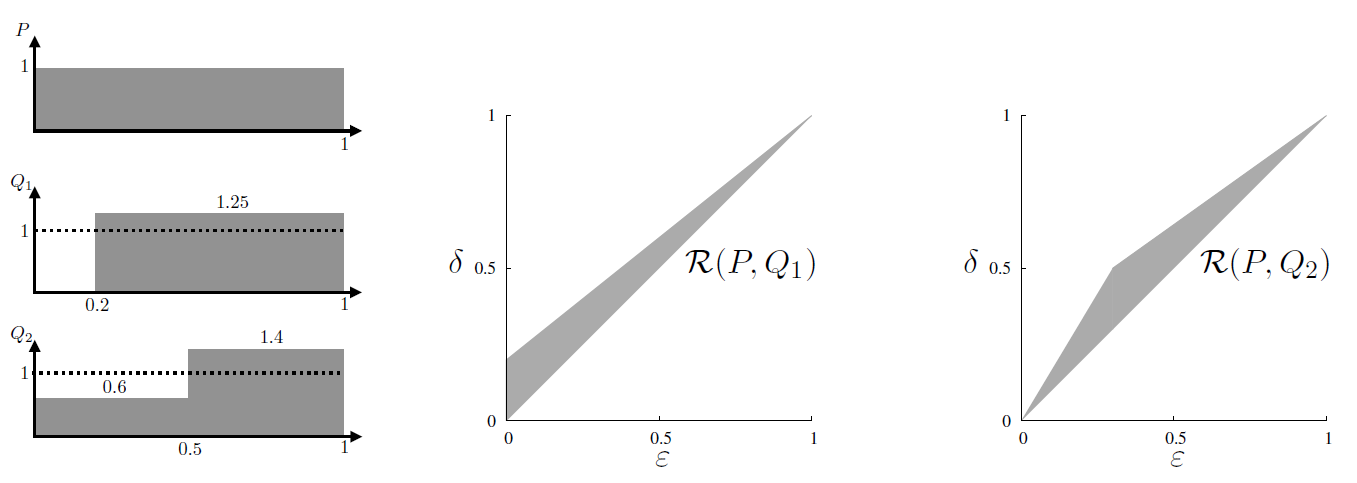}
		\caption{This figure is excerpted from \cite{pacgan}. The left panel shows the probability density functions of $ P, Q_{1}, \text{ and } Q_{2}.$ The middle panel shows the mode collapse region $ R(P,Q_{1}).$ The right panel shows the mode collapse region $ R(P,Q_{2}).$}
		\label{toyexample}
	\end{figure}

	However, if we draw m samples $ (x_{1}, ..., x_{m}) $ from distribution $ P ,$ and consider $ X = (x_{1}, ..., x_{m}) $ as a single variable, then $ X $ has a probability distribution of $ P^{m} .$ Similarly, drawing m samples at the same time from probability distribution $ Q_{1} $ and $ Q_{2} $ obtain a probability distribution $ Q_{1}^{m} $ and $ Q_{2}^{m} $ of the m samples. We call the number of samples m drawn from a probability distribution the \textbf{packing degree}. Thus, the mode collapse regions between the target distribution and generated distributions with packing degree m are $ R(P^{m},Q_{1}^{m}) $ and $ R(P^{m}, Q_{2}^{m}) $, and the total variation distance between the target distribution and generated distributions with packing degree m are $ d_{TV}(P^{m},Q_{1}^{m}) $ and $ d_{TV}(P^{m}, Q_{2}^{m}) $. Figure \ref{toyexamplem=5} shows the evolution of mode collapse regions $ R(P^{m},Q_{1}^{m}) $ and $ R(P^{m}, Q_{2}^{m}) $ and total variation distances $ d_{TV}(P^{m},Q_{1}^{m}) $ and $ d_{TV}(P^{m}, Q_{2}^{m}) $ with packing degree m = 1 to 5.
	
	\begin{figure}[H]
		\centering
		\includegraphics[scale=1]{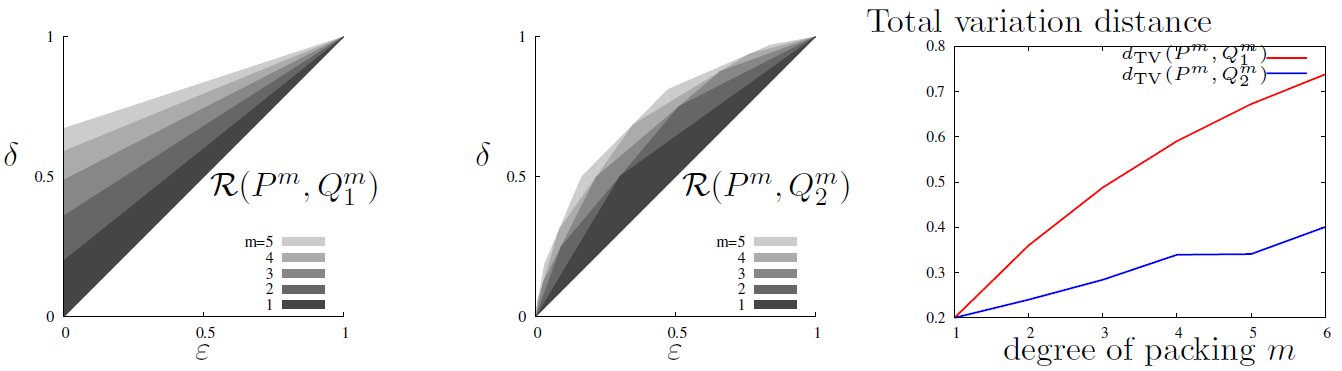}
		\caption{This figure is excerpted from \cite{pacgan}. From packing degree m = 1 to 5, the left panel shows the mode collapse regions of $ R(P^{m}, Q_{1}^{m}) $, the middle panel shows the mode collapse regions $ R(P^{m}, Q_{2}^{m}) $, and the right panel shows $ d_{TV}(P^{m},Q_{1}^{m}) $ and $ d_{TV}(P^{m}, Q_{2}^{m}) $ from m = 1 to 6.}
		\label{toyexamplem=5}
	\end{figure} 

	From Figure \ref{toyexamplem=5}, we can see that although $ (P,Q_{1}) $ and $ (P,Q_{2}) $ have equal area of mode collapse regions and total variation distances, $ R(P^{m},Q_{1}^{m}) $ and $ d_{TV}(P^{m},Q_{1}^{m}) $ hiked as packing degree m increased, while $ R(P^{m},Q_{2}^{m}) $ and $ d_{TV}(P^{m},Q_{2}^{m}) $ increase gradually with the packing degree. For computational efficiency, we choose to measure the severity of mode collapse using total variation distance.\\
	
	However, without packing, generated distributions can have the same total variation distance to the target distribution with very different mode collapse regions. Consider any pair of target distribution $ P $ and generated distribution $ Q $ $ (P,Q) $ such that $ d_{TV}(P,Q) = \tau $, where $ 0\leq \tau \leq 1 .$ Under the same total variation distance, the mode collapse region between $ P $ and $ Q $ can be drastically different.In Figure \ref{3bounds}, there are three mode collapse regions: $ R(P_\text{in},Q_\text{in}) $ in red, $ R(P, Q) $ in blue, and $ R(P_\text{out}, Q_\text{out}) $ in green. $P_\text{in} $ is the target distribution and $ Q_\text{in} $ is the generated distribution such that $ R(P_\text{in}, Q_\text{in}) = \displaystyle\min_{P, Q} R(P,Q) $ subjected to $ d_{TV}(P,Q) = \tau $. $ P_\text{out} $ is the target distribution and $ Q_\text{out} $ is the generated distribution such that $ R(P_\text{out}, Q_\text{out}) = \displaystyle\max_{P, Q} R(P,Q) $ subjected to $ d_{TV}(P,Q) = \tau $.
	The red line on the upper boundary shared by these three regions has a slope of 1 and is parallel to the line $ \varepsilon = \delta $. 	
	
	\begin{figure}[H]
		\centering
		\includegraphics[scale=0.5]{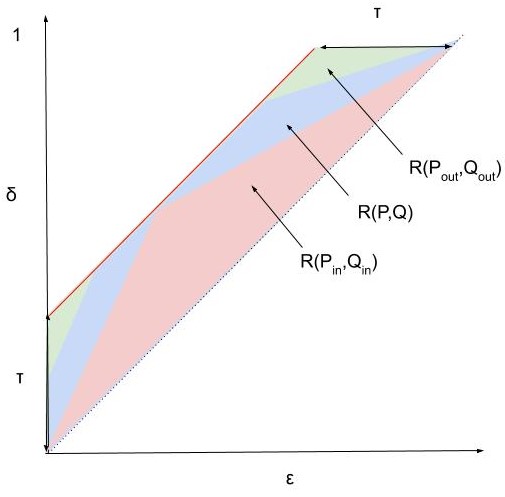}
		\caption{Under the same total variation distance of $ \tau $ between target and generated distribution, the minimum mode collapse region $ R(P_\text{in}, Q_\text{in}) $ is shaded in red, the maximum mode collapse region $ R(P_\text{out}, Q_\text{out}) $ is shaded in green, and $ R(P,Q) $, shaded in blue, is sandwiched between the minimal and maximal mode collapse regions.}
		\label{3bounds}
	\end{figure}
	
	From Figure \ref{3bounds}, we can see that $ d_{TV}(P_\text{in},Q_\text{in}) = d_{TV}(P,Q) = d_{TV}(P_\text{out},Q_\text{out})$, even though $ R(P_\text{in},Q_\text{in}) \subseteq R(P,Q) \subseteq R(P_\text{out}, Q_\text{out}) $. To see the effect of packing m samples on total variation distances, we want to find $ d_{TV}(P_\text{in}^{m},Q_\text{in}^{m}) $, the minimum of $ d_{TV}(P^{m},Q^{m}) $, and $ d_{TV}(P_\text{out}^{m},Q_\text{out}^{m}) $, the maximum of $ d_{TV}(P^{m},Q^{m}) $. Theorem \ref{dTVbounds} states that there is a substantial difference between $ \displaystyle\min_{P,Q} d_{TV}(P^{m},Q^{m}) $ and $ \displaystyle\max_{P,Q} d_{TV}(P^{m},Q^{m}) $. Such difference indicates that $ d_{TV}(P^{m},Q^{m}) $ is a more effective measurement of mode collapse than $ d_{TV}(P,Q) $.\\
	
	\newtheorem{thm}{Theorem}
	\begin{thm}[This is Theorem 3 in \cite{pacgan}]
		\label{dTVbounds}
		
		For $ 0 \leq \tau \leq 1 $ and positive integer m, and for any pair of target and generated probability distributions $ (P,Q) $ such that $ d_{TV}(P,Q)=\tau $, 
		
		\[\displaystyle\min_{P,Q} d_{TV}(P^{m},Q^{m}) = \displaystyle\min_{0\leq\alpha\leq 1-\tau} d_{TV}(P(\alpha)^{m},Q(\alpha, \tau)^{m}) \] 
		
		where $ P(\alpha) $ and $ Q(\alpha, \tau) $ are the m\textsuperscript{th} order product distributions of random variables distributed as 
		\[P(\alpha) = [1-\alpha, \alpha]\]
		\[Q(\alpha, \tau) = [1-\alpha-\tau, \alpha + \tau];\]
	
		\[\text{and } \displaystyle\max_{P,Q} d_{TV}(P^{m},Q^{m}) = 1-(1-\tau)^{m}.\]
		
	\end{thm} 
	
	To prove Theorem \ref{dTVbounds}, the authors of PacGAN established a few propositions and theorems that will later be applied in the proof. They applied binary hypothesis testing in the context of GAN discriminator. Through binary hypothesis testing, they are able to apply data processing inequality onto mode collapse regions, thus proving an inequality between mode collapse regions of probability distributions implies an inequality between total variation distances of these distributions \cite{pacgan}.\\
	
	Binary hypothesis testing consists of experiments based upon two hypotheses. The null hypothesis is defined as $ h = 0$. The alternate hypothesis, i.e. rejection of null hypothesis, is defined as $ h=1 $. In the experiment, an observation is made. Let $ X $ be a random variable that denotes the outcome of observation. The outcome of observation depends upon a rejection set $ S_{r} $ drawn from the outcome range $ \mathcal{X} $, such that if $ X\in S_{r} $, then the null hypothesis is rejected. We define \textbf{false positive rate} (FPR) as the probability of null hypothesis is true but rejected during observation, i.e. $ \text{FPR} = \mathbb{P}(X\in S_{r}|h=0) $. We define true positive rate as the probability of null hypothesis is false and rejected during observation, i.e. $ \text{TPR} = \mathbb{P}(X\in S_{r}|h=1) $. We establish a 2D plane with FPR as x-axis and TPR as y-axis, we define the \textbf{hypothesis testing region} in this plane.\\
	
	\begin{dfn}[This is an expansion of the definition of hypothesis testing region in p.27 in \cite{pacgan}]
		\label{hregion}
		The hypothesis testing region $ R(h, X) $ is the convex hull of the region of all points \\
		${(\mathbb{P}(X\in S_{r}|h=0), \mathbb{P}(X\in S_{r}|h=1))}$ for all possible rejection sets $ S_{r} $,\\ i.e. 
		${R(h, X) = \mathrm{conv}((\mathbb{P}(X\in S_{r}|h=0), \mathbb{P}(X\in S_{r}|h=1))}$.
	\end{dfn}
	 
	We can apply binary hypothesis testing to the context of a GAN discriminator. We consider the null hypothesis $h = 0$ as the data sample is from the generated distribution $ Q $, and consider the alternate hypothesis $h = 1$ as the data sample is from the target distribution $ P $ . The discriminator receives input sample $ X $, and decides whether $ X $ is a real data sample from distribution $ P $ or a synthetic sample from distribution $ Q $. We make the decision based on the rejection set $ S_{r} $, such that if $ X\in S_{r} $, then the discriminator labels the data sample as real data. FPR is the probability the data is created by the generator but the discriminator mistakenly labels it as real data. TPR is the probability that the data is created by the generator and the discriminator correctly labels it as synthetic. We assume that FPR $ < $ TPR, because in a large set samples, an untrained discriminator which randomly assigns labels to data samples would on average have equal number of mistakenly labeled samples and correctly labeled samples. In reality, the discriminator in GAN is trained and has superior performance than random assignments.\\
	
	Recall from Definition \ref{modecollapsedef} and the definition of mode collapse region that for $ 0\leq\varepsilon<\delta\leq 1 $, $ (P,Q) $ has $ (\varepsilon, \delta) $ mode collapse if there exist a set $ S $ such that $ P(S)\geq\delta $ and $ Q(S)\leq\varepsilon $. On a 2D plane where $ \varepsilon $ is the x-axis and $ \delta $ is the y-axis, $ R(P,Q) $ is the convex hull of the region of all points $(\varepsilon, \delta)$ such that $ (P,Q) $ has $(\varepsilon, \delta)$ mode collapse. The hypothesis testing region is the convex hull of the region of all points $ (\mathbb{P}(X\in S_{r}|h=0), \mathbb{P}(X\in S_{r}|h=1)) $ for all possible rejection sets $ S_{r} $. Since we have established the assumption $ 0\leq \text{FPR} \leq \text{TPR} \leq 1 $, we can find a set $ S = S_{r} $ so that $ \mathbb{P}(X\in S_{r}|h=0) = \varepsilon $ and $ \mathbb{P}(X\in S_{r}|h=1) = \delta $. By finding such set $ S $, we can establish a bijection between mode collapse region and hypothesis testing region. Using Definition \ref{modecollapsedef}, \ref{R(P,Q)def}, and \ref{hregion}, the authors of PacGAN established a bijection between mode collapse region and hypothesis testing region \cite{pacgan}.\\
	
	By the one-to-one correspondence, $ P(S)=\mathbb{P}(X\in S_{r}|h=1) $ and $ Q(S)=\mathbb{P}(X\in S_{r}|h=0) $ for $ S_{r} \subset [0,1] $. Similarly, for a different pair of distribution $ (P',Q') $, we denote the random variable that decides to reject the null hypothesis $ h' $ or not as $ X' $. We can find a $ S' $ from the common domain of $ P' $ and $ Q' $ such that  $ P'(S')= \mathbb{P}(X'\in S'_{r}|h'=1)$ and $ Q'(S')= \mathbb{P}(X'\in S'_{r}|h'=0)$ for $S'_{r} \subset [0,1] $.\\
	
	We have established the bijection between hypothesis testing region and mode collapse region. Now we draw connection between mutual information and hypothesis testing region, so that later we can introduce a powerful theorem, data processing inequality. Higher mutual information between hypothesis $ h $ and random variable $ X $ leads to higher TPR and lower FPR in the hypothesis testing region. High TPR and low FPR correspond to a larger area of mode collapse with a large total variation distance. Data processing inequality states that the more a random variable is processed, there will be less mutual information between the hypothesis and that variable, which leads to a smaller hypothesis testing region and a smaller mode collapse region. \\
	
	Before we introduce data processing inequality, we need to first define a Markov chain, as we will need Markov chain as a condition in the theorem later. \\
	
	\begin{dfn}[\cite{cover}]
		Random variable $X, Y, Z$ form a Markov chain $X-Y-Z$ if the conditional distribution of $Z$ depends only on $Y$ and is conditionally independent of $X$, i.e. $X, Y, Z$ form a Markov chain $X-Y-Z$ if the joint probability density function can be written as $ \mathbb{P}(X,Y,Z) = \mathbb{P}(X)\mathbb{P}(Y|X)\mathbb{P}(Z|Y) $. 
	\end{dfn}
	
	Now we apply the concept of Markov chain to random variables $ h $, $ X $ and $ X' $ in binary hypothesis testing. Since $ P(S) = \mathbb{P}(X\in S_{r}|h=1) $ and  $ Q(S) = \mathbb{P}(X\in S_{r}|h=0)$ from the bijection between hypothesis testing region and mode collapse region, $ X $ is conditionally dependent on $h$. $ X' $ is conditionally dependent on hypothesis $ h' $, and $ X' $ is conditionally independent of $ h $. If we can find a function $ f $ such that $ f(P,Q) = (P',Q') $, then $ f(\mathbb{P}(X\in S_{r}|h=1), \mathbb{P}(X\in S_{r}|h=0)) = (\mathbb{P}(X'\in S'_{r}|h'=1), \mathbb{P}(X'\in S'_{r}|h'=0)) $. This implies that function f establishes a conditional dependence between the probability distribution of $ X$ and the probability distribution of $ X' $. Thus, we can argue that $ h, X, X' $ form a Markov chain $ h-X-X' $. \\
	
	After we defined the Markov chain, we are ready to present the data processing inequality. Later we will apply data processing inequality in the context of target distribution $ P $ and generated distribution $ Q $ in a GAN. 
	
	\begin{thm}[Data Processing Inequality \cite{cover}]
		Let $ X-Y-Z $ be a Markov chain, then $ I(X;Z)\leq I(X;Y) $ where $ I(X;Y) $ denotes the mutual information between $ X $ and $ Y $ .
	\end{thm}
		
	\begin{proof}
		\begin{align*}
			I(X;Y)&=\sum\limits_{x,y} p(x,y)\log\frac{p(x,y)}{p(x)p(y)}\\ &=\sum\limits_{x,y}p(x,y)\log\frac{p(x|y)p(y)}{p(x)p(y)}\\
			&=\sum\limits_{x,y}p(x,y)\log\frac{p(x|y)}{p(x)}\\
			&=-\sum\limits_{x,y}p(x,y)\log p(x) +\sum\limits_{x,y}p(x,y)\log p(x|y)\\
			&=-\sum\limits_{x}p(x)\log p(x) + \sum\limits_{x,y}p(x,y)\log p(x|y)\\
			&=H(X)-H(X|Y).
		\end{align*}
		Because $ X-Y-Z $ is a Markov chain, $ H(X)-H(X|Y)= H(X)-H(X|Y,Z) $. Because conditioning reduces entropy, $ H(X)-H(X|Y,Z)\geq H(X)-H(X|Z)=I(X;Z) $.
	\end{proof}
	
	Applying data processing inequality in the hypothesis testing region, we can see that an increase in mutual information $ I(h;X) $ results in higher TPR and lower FPR rate. We have established the assumption that $ FPR < TPR $ because in the worst case scenario, randomly assigning values to X will on average achieve $ FPR = TPR $. High TPR and low FPR create a large hypothesis testing region. Recall that there exists a bijection between hypothesis testing region $ R(h, X) $  and mode collapse region $ R(P,Q) $. If $ h-X-X' $ is a Markov chain, then $ I(h;X')\leq I(h;X) $ implies $ R(h, X') \subseteq R(h, X) $. And $ R(h, X') \subseteq R(h, X) $ is equivalent to $ R(P',Q')\subseteq R(P,Q).$ \\
	
	Using data processing inequality, we are able to prove that if there exists a Markov chain $ h-X-X' $ between random variables $ h, X $ and $ X' $, then $ R(P',Q') \subseteq R(P,Q). $ However, in the mode collapse region, we currently have $ R(P_\text{in}, Q_\text{in}) \subseteq R(P,Q) \subseteq R(P_\text{out}, Q_\text{out}) $, and we want to show that \\
	$ d_{TV}(P_\text{in}^{m}, Q_\text{in}^{m}) \leq d_{TV}(P^{m},Q^{m}) \leq d_{TV}(P_\text{out}^{m}, Q_\text{out}^{m}) $. To do so, the reverse of data processing inequality needs to hold true, i.e. if $ R(P',Q') \subseteq R(P,Q) $, then there exists a Markov chain $ h-X-X'$. \\
	
	In 1953, David Blackwell published a celebrated result in his paper \textit{Equivalent Comparison of Experiments}: the reverse of data processing inequality is also true \cite{blackwell}.  Let $ B $ and $ C $ be Markov matrices with conditional probabilities of displaying a signal given the current state. Given $ C $ is more informative than $ B $, then there exists a Markov matrix $ M $ such that $ B=CM $. Blackwell's proof is long and convoluted. In this paper we will use a more concise proof by Moshe Lesno and Yishay Spector \cite{blackwellproof}.\\
	
	Denote the set of states in a Markov chain as $ S=\{s_{1},...,s_{n}\} $ and the probabilities associated with these states by $ p=(p_{1},...,p_{n}) $ with $\sum_{i=1}^{n} p_{i}=1$. Denote the signals observed to be $ Y=\{y_{1},...,y_{q}\} $. Denote $ B \in \mathbb{R}^{n\times q}$ to be the Markov matrix that indicates the conditional probabilities of signal $ y $ is displayed given that the current state is $ s $. So $ B_{ij} $ is the probability for a given state $ s_{i} $ that signal $ y_{j} $ will be displayed. Denote the set $ A = \{a_{1},...,a_{r}\} $ to be a set of actions taken. Let the payoff function $ U: A\times S\to \mathbb{R} $ associate with a pair of action and state. $ U $ can be written as a matrix with $ u_{ij} = U(a_{i}, s_{j}) $. Here $ u_{ij} $ represents the payoff gained when action $ a_{i} $ is taken and the state transitions to $ s_{j} $. Denote $ D\in \mathbb{R}^{q\times r} $ to be the decision matrix. $ D$ is a Markov matrix with $ d_{ij} $ as the probability that the decision maker takes action $ a_{j} $ on observing signal $ y_{i} $. Let $ P $ to be the square diagonal matrix containing probability $ p_{i} $ on diagonal entry $ P_{ii} $ . Then the expected payoff is $ \tr(BDUP) $. Denote the maximum expected payoff to be $ F(B,U,P)=\displaystyle \max_{D} \tr(BDUP) $. \\
	
	We have stated the matrices and operations necessary for the proof of reverse data processing inequality. Now we define how much information a matrix contains using maximum expected payoff. \\
	
	\begin{dfn}[This is Definition 1 in \cite{blackwellproof}]
		\label{moreinformative}
		B is more informative than C, denoted as $ C\subseteq B $, \newline
		if $F(B,U,P) \geq F(C,U,P) $.
	\end{dfn}
	
	To facilitate the proof, we also introduce two propositions which concern inner product and trace operations. Later in the proof of reverse data processing inequality, we will directly apply the result of Proposition \ref{innerproduct} and \ref{trinnerproduct}.\\
	
	\newtheorem{prop}{Proposition}
	\begin{prop}[This is Proposition 1 in \cite{blackwellproof}]
		\label{innerproduct}
		$ A,B $ and $ C $ are square matrices with $ A,B,C\in \mathbb{R}^{n \times n} $. Then
		$ \langle AB,C \rangle =\langle A^{T},CB^{T}\rangle $. 
	\end{prop}
	
	\begin{proof}
		\[\langle AB,C \rangle=\tr(B^{T}A^{T}C) = \tr(A^{T}CB^{T})=\langle A^{T}, CB^{T} \rangle.\]
	\end{proof}

	\begin{prop}[This is Proposition 2 in \cite{blackwellproof}]
		\label{trinnerproduct}
		$ A,B,C $ and $ D $ are matrices with $ A \in \mathbb{R}^{n\times q}, B \in\mathbb{R}^{q\times r}$, $ C \in\mathbb{R}^{r\times n}, D \in\mathbb{R}^{n\times n} $. Then $ \tr(ABCD)= \langle ABC,D \rangle. $ 
	\end{prop}

	\begin{proof}
		\[\tr(ABCD)= \sum_{i}(ABC)_{ii}D_{ii} =  \sum_{i,j}(ABC)_{ij}D_{ij} = \langle ABC,D \rangle.\]
	\end{proof}

	With all definitions and propositions in place, now we are ready to introduce reverse data processing inequality. In this proof, the notation follows those defined prior to Definition \ref{moreinformative}.

	\begin{thm}[Reverse Data Processing Inequality \cite{blackwell}\cite{blackwellproof}]
		\label{rdpi}
		$ B $ is more informative than $ C $, i.e. $ C\subseteq B $ if and only if there exists a Markov matrix $ M $ such that $ C = BM $. 
	\end{thm}

	\begin{proof}
		If $ C = BM $, then $ \tr(CDUP)=\tr(BMDUP) $. Because $ M $ is a Markov matrix, $ \displaystyle\max_{D} \tr(CDUP) = \displaystyle\max_{D}(BMDUP) \leq \displaystyle\max_{D}\tr(BDUP),$ i.e. $ C \subseteq B $.\\
		
		Suppose for every Markov matrix $ M $ , $ C\neq BM $. Let $ S=\{A|\exists M$ that is a Markov matrix such that $A=BM\}$. Suppose $ M=cM_{1}+(1-c)M_{2} $ where $ 0\leq c \leq 1 $. Let $ m_{ij} =(1-\lambda)m_{1ij}+\lambda m_{2ij}$ with $ 0 \leq\lambda\leq 1 $. Since $ m_{1ij},m_{2ij} >0 $,
		\begin{align*}
			\sum_{j=1}^{n}m_{ij}&= \sum_{j=1}^{n}((1-\lambda)m_{1ij}+\lambda m_{2ij})\\
			&=(1-\lambda)\sum_{j=1}^{n} m_{1ij} + \lambda\sum_{j=1}^{n} m_{2ij}\\ &= 1-\lambda+\lambda = 1.
		\end{align*}
		 Therefore the set of Markov matrix is convex. Then $ P(cM_{1}+(1-c)M_{2})=PM\in S. $ Thus $ S $ is convex.\\
		
		$ S $ is the image of the set of Markov matrices under the continuous map $ f(M)=BM $. Since $ m_{ij}\geq 0, \forall i,j $ and $$\sum_{j=1}^{n} m_{ij} = 1,$$ the set contains all of its limit points. The set of Markov matrices is bounded because $ 0\leq m_{ij}\leq 1, \forall i,j.$ Therefore the set of Markov matrices is compact. The image of a compact set on a continuous map is compact. A compact set in $ \mathbb{R}^{q\times q} $ is closed. So the set S is closed. \\
		
		Using similar reasoning, we can deduce that C is compact. $ C $ and $ S $ are disjoint, nonempty sets that are closed and convex, and $ C $ is compact. By the hyperplane separation theorem, there exists a matrix $ \hat{U} $ such that $ \forall M,$ $\langle(BM)^{T},\hat{U}\rangle < \langle C^{T}, U\rangle.$ \\
		
		Let the payoff matrix $ U $ be $ U^{T} = P^{-1}\hat{U} $. Using Proposition \ref{innerproduct} and \ref{trinnerproduct}, we have
		\begin{align*}
			\tr(BMUP) &= \langle BMU, P \rangle \\
			&= \langle BM\hat{U}^{T}(P^{-1})^{T}, P \rangle \\
			&= \langle (BM)^{T}, P(\hat{U}^{T}(P^{-1})^{T})^{T} \rangle \\
			&= \langle (BM)^{T}, PP^{-1}\hat{U} \rangle\\ 
			&= \langle(BM)^{T}, \hat{U}\rangle.
		\end{align*} 
		
		Using the same computations, we can obtain $ \tr(CUP) = \langle C^{T}, \hat{U}\rangle $. From $\langle(BM)^{T},\hat{U}\rangle < \langle C^{T}, U\rangle $, we can deduce that $ \tr(BMUP) < \tr(CUP)$.\\ 
		
		Therefore $\displaystyle\max_{M}\tr(BMUP) < \tr(CIUP) < \displaystyle\max_{M}\tr(CMUP).$ \\
		
		Thus we have $ B \nsubseteq C. $
	\end{proof}

	In Theorem \ref{rdpi}, we know that $ CDU $ is the payoff matrix when state $ s_{i} $ transitions to state $ s_{j} $. $ \tr(CDUP)=\langle CDU, P \rangle $, which denotes the average payoff when we transition from state $ s_{i} $ to state $ s_{j} $. In the hypothesis testing region, if the average payoff is greater, then the hypothesis testing region is greater. Now we apply Theorem 2 to the GAN discriminator context. Recall that for random variable $ X $, target distribution $ P(S) = \mathbb{P}(X\in S_{r}|h=1) $ and generated distribution $ Q(S) = \mathbb{P}(X\in S_{r}|h=0)$. For random variable $ X' $, target distribution $ P'(S) = \mathbb{P}(X'\in S'_{r}|h'=1) $ and generated distribution $ Q'(S) = \mathbb{P}(X'\in S'_{r}|h'=0)$. If there exists an inequality $ R(P',Q')\subseteq R(P,Q) $ between mode collapse regions, then by the bijection between mode collapse region and hypothesis testing region, the inequality between hypothesis testing regions $ R(h, X')\subseteq R(h, X) $. Let $ B $ denote the Markov matrix whose entries are the conditional probabilities of observing $ X $ given hypothesis $ h $. Let $ C $ denote the Markov matrix whose entries are the conditional probabilities of observing $ X' $ given $ X $. Because $ R(h, X')\subseteq R(h, X) $, $ B $ is more informative than $ C $. Therefore, there exists a Markov matrix $ M $ such that $ C = BM $. Thus, there exists a Markov chain $ h-X-X' $. From the Markov chain $ h-X-X' $, we can deduce that $ d_{TV}(P_\text{in}^{m}, Q_\text{in}^{m}) \leq d_{TV}(P^{m},Q^{m}) \leq d_{TV}(P_\text{out}^{m}, Q_\text{out}^{m}) $.\\
	
	Now we are going to use reverse data processing inequality to prove Theorem \ref{dTVbounds}.
	
	\begin{proof}[Proof to Theorem \ref{dTVbounds}](This is Proof of Theorem 3 in \cite{pacgan})

	 From Figure \ref{3bounds}, we can see that
		$$ R_\mathrm{inner}(P,Q)\subseteq R(P,Q)\subseteq R_\mathrm{outer}(P,Q). $$	
	 thus we have 
		$$ R_\mathrm{inner}(P^{m},Q^{m})\subseteq R(P^{m},Q^{m})\subseteq R_\mathrm{outer}(P^{m},Q^{m}) .$$
	By reverse data processing inequality, 
		$$ \displaystyle\min_{0\leq\alpha\leq 1-\tau} d_{TV}(P_\mathrm{inner}(\alpha)^{m},Q_\mathrm{inner}(\alpha,\tau)^{m}) \leq d_{TV}(P^{m},Q^{m}) \leq d_{TV}(P_\mathrm{outer}(\tau)^{m},Q_\mathrm{outer}(\tau)^{m}). $$
		
	To think in terms of probability, the probability that a point is out of the region covered by total variation distance is $ 1-\tau $. The probability of a point in the  $m^{th}$  dimension is out of the region covered by total variation distance is $ (1-\tau)^{m} $. Therefore, the probability of the point in the $ m^{th} $ dimension is in the region covered by total variation distance is $ 1-(1-\tau)^{m} $. So $$ d_{TV}(P_\mathrm{outer}(\tau)^{m},Q_\mathrm{outer}(\tau)^{m})=1-(1-\tau)^{m}. $$
	
	To achieve the minimum of $ d_{TV}(P^{m},Q^{m}), $ we need to find the optimal $ \alpha $ that minimizes $ d_{TV}(P^{m},Q^{m}),$ with $ \alpha $ denotes the horizontal distance between the point farthest from the line $ \varepsilon=\delta $ and $ \tau $ in Figure \ref{3bounds}.
	
	The lower bound for $ d_{TV}(P^{m},Q^{m}) $ is $$ \displaystyle\min_{0\leq\alpha\leq 1-\tau} d_{TV}(P_\mathrm{inner}(\alpha)^{m},Q_\mathrm{inner}(\alpha,\tau)^{m}) $$ with $ P_\mathrm{inner}=[1-\alpha,\alpha] $ and $ Q_\mathrm{inner}=[1-\alpha-\tau,\alpha+\tau] $.
	\end{proof}

	 \vspace{3mm}
	
	\section{VEEGAN}
	\subsection{Structure of VEEGAN}
	To mitigate mode collapse, VEEGAN introduces a new neural network, a reconstructor, along with the generator and the discriminator in its architecture. The generator maps inputs drawn from a Gaussian distribution to the synthetic data; the discriminator maps data to binary labels of real or synthetic. The reconstructor maps data to outputs that form a Gaussian distribution. The reconstructor has two objectives: 1. Map all data in the data space to outputs in a Gaussian distribution; 2. Act as the inverse of the generator. If the reconstructor maps all data in the data space to outputs in a Gaussian distribution and acts as an inverse of the generator at the same time, it in turn encourages the generator to map from inputs drawn from a  Gaussian distribution to synthetic data that match the target data distribution, thus resolving mode collapse \cite{veegan}. Figure \ref{VEEGANgraph1} and \ref{VEEGANgraph2} shows the structure of VEEGAN, a system with a generator, a discriminator, and a reconstructor.\\
	
	\begin{figure}[H]
		\centering
		\begin{subfigure}{0.5\textwidth}
			\begin{tikzpicture}
			\begin{axis}
			[every axis plot post/.append style= {mark=none,domain=-5:5,samples=50,smooth}, % All plots: from -5:5, 50 samples, smooth, no marks
			axis x line*=bottom, % no box around the plot, only x and y axis
			axis y line* = none,
			enlargelimits=upper,% extend the axes a bit to the right and top
			axis line style={draw=none},
			tick style={draw=none},
			xticklabels={,,},
			yticklabels={,,},
			clip = false] 
			\addplot [color = black] {gauss(0,1)};
			\addplot [->] coordinates {(-5,0) (6,0)} node[right]{$ z $};
			\node at (800, 40) {$ G_{\gamma} $};
			\node at (850, 150) {$ F_{\theta} $};
			\node at (100, 20){$ p_{0}(z) $};
			\node at (90, 130) {$ p(x) $};
			\addplot [color = green] [->] coordinates {(0,0) (1.8,1)};
			\addplot [color = green] [->] coordinates {(0.8,0) (2.4,1)};
			\addplot [color = green] [->] coordinates {(0.4,0) (2.2,1)};
			\addplot [color = green] [->] coordinates {(2,0) (3,1)};
			\addplot [color = green] [->] coordinates {(-0.5,0) (1,1)};
			\addplot [color = green] [->] coordinates {(1,1) (-2,2)};
			\addplot [color = green] [->] coordinates {(1.8,1) (-1.5,2)};
			\addplot [color = green] [->] coordinates {(2.2,1) (-0.5,2)};
			\addplot [color = green] [->] coordinates {(2.4,1) (0,2)};
			\addplot [color = green] [->] coordinates {(3,1) (1.5,2)};
			\addplot [color = black]{gauss(-2,1)+gauss(2,1)+1};
			\addplot [->] coordinates { (-5,1) (6,1) } node[right]{$ x $};
			\addplot [color = green] {gauss(0,1)+2};
			\addplot [color = violet] {gauss(-3,0.5) + gauss(1,1) +2};
			\addplot [color = violet] [->] coordinates {(-2,1) (-3,2)};
			\addplot [color = violet] [->] coordinates {(-3,1) (-3.5,2)};
			\addplot [color = violet] [->] coordinates {(-1,1) (-2.5,2)};
			\addplot [color = violet] [->] coordinates {(2.7,1) (3,2)};
			\addplot [color = violet] [->] coordinates {(2,1) (0,2)};
			\addplot [color = violet] [->] coordinates {(2.5,1) (0.3,2)};
			\addplot [->] coordinates { (-5,2) (6,2) } node[right]{$ \hat{z} $};
			\end{axis}
			\end{tikzpicture}
			\caption{$ F_{\theta} $ is trained to be the inverse of $ G_{\gamma} $.}
			\label{VEEGANgraph1}
		\end{subfigure}%
		\begin{subfigure}{0.5\textwidth}
			\begin{tikzpicture}
			\begin{axis}
			[every axis plot post/.append style= {mark=none,domain=-5:5,samples=50,smooth}, % All plots: from -5:5, 50 samples, smooth, no marks
			axis x line*=bottom, % no box around the plot, only x and y axis
			axis y line* = none,
			enlargelimits=upper,% extend the axes a bit to the right and top
			axis line style={draw=none},
			tick style={draw=none},
			xticklabels={,,},
			yticklabels={,,},
			clip = false] 
			\node at (800, 40) {$ G_{\gamma} $};
			\node at (850, 150) {$ F_{\theta} $};
			\node at (100, 20){$ p_{0}(z) $};
			\node at (90, 130) {$ p(x) $};
			\addplot [color = black] {gauss(0,1)};
			\addplot [->] coordinates { (-5,0) (6,0) } node[right]{$ z $};
			\addplot [color = green] [->] coordinates {(-1,0) (1.5,1)};
			\addplot [color = green] [->] coordinates {(0,0) (1.8,1)};
			\addplot [color = green] [->] coordinates {(0.8,0) (2.4,1)};
			\addplot [color = green] [->] coordinates {(0.4,0) (2.2,1)};
			\addplot [color = green] [->] coordinates {(2,0) (3,1)};
			\addplot [color = green] [->] coordinates {(-0.5,0) (1,1)};
			\addplot [color = black]{gauss(-2,1)+gauss(2,1)+1};
			\addplot [->] coordinates { (-5,1) (6,1) } node[right]{$ x $};
			\addplot [color = violet] [->] coordinates {(-1,1) (0,2)};
			\addplot [color = violet] [->] coordinates {(-2,1) (-1,2)};
			\addplot [color = violet] [->] coordinates {(-2.5,1) (-1.5,2)};
			\addplot [color = violet] [->] coordinates {(-3,1) (-2,2)};
			\addplot [color = violet] [->] coordinates {(1,1) (-0.5,2)};
			\addplot [color = violet] [->] coordinates {(2,1) (1,2)};
			\addplot [color = violet] [->] coordinates {(2.5,1) (0.5,2)};
			\addplot [color = violet] [->] coordinates {(3,1) (2,2)};
			\addplot [color = violet] {gauss(0,1)+2};
			\addplot [color = green] {gauss(1,0.75) +2};
			\addplot [color = green] [->] coordinates {(2,1) (1,2)};
			\addplot [color = green] [->] coordinates {(1,1) (0.5,2)};
			\addplot [color = green] [->] coordinates {(2.5,1) (1.5,2)};
			\addplot [color = green] [->] coordinates {(3,1) (2,2)};
			\addplot [->] coordinates { (-5,2) (6,2) } node[right]{$ \hat{z} $};
			\end{axis}
		\end{tikzpicture}
	\caption{$ F_{\theta} $ is trained to map data to a Gaussian distribution.}
	\label{VEEGANgraph2}
		\end{subfigure}
	\caption{The two figures are excerpted from Figure 1 in \cite{veegan}. These figures illustrate how a reconstructor $ F_{\theta} $ helps detect mode collapse in generator $ G_{\gamma}$.}
	\label{VEEGANgraph}	
	\end{figure}
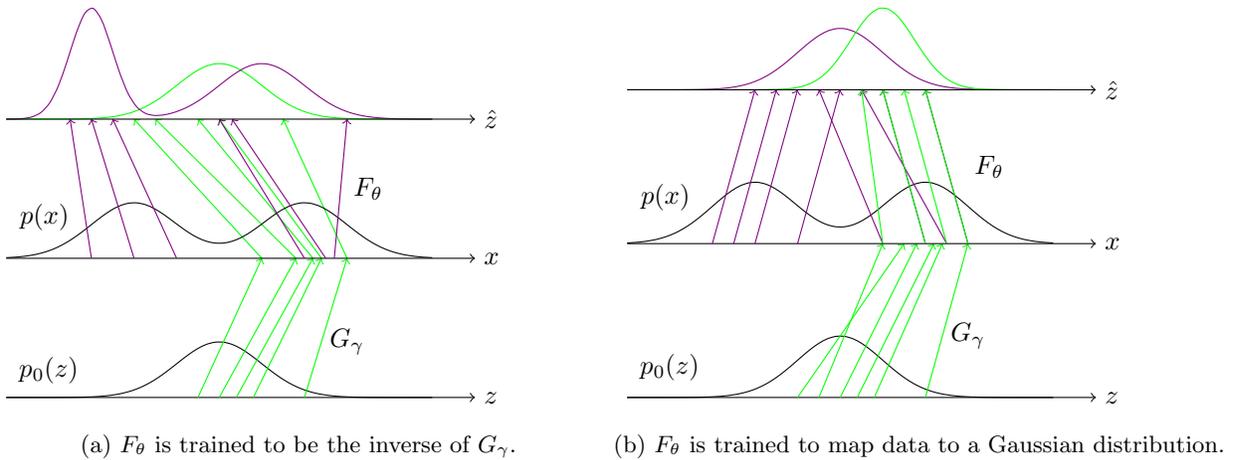		
	
Let $ z $ denote Gaussian noise vectors which are inputs to the generator. Let $ p_{0}(z) $ denote the distribution of input to generator. Let $ x $ denote the data provided to the VEEGAN architecture. Let $ p(x) $ denote the target data distribution. Let $ \hat{z} $ denote noise vectors that are outputs of the reconstructor. Let $ \gamma $ denote the weights of the neural network in the generator, and let $ G_{\gamma} $ denote the generator, which maps Gaussian noise vectors to synthetic data. Let $ \theta $ denote the weights of the neural network in the reconstructor, and let $ F_{\theta} $ denote the reconstructor, which maps data to Gaussian noise vectors. In Figure \ref{VEEGANgraph1} and \ref{VEEGANgraph2}, $ p_{0}(z) $ is a Gaussian distribution. $ p(x) $ is the target data distribution. The purple distribution on the top is the output of reconstructor $ F_{\theta}(x) $ mapped from the target data distribution $ p(x) $. The green distribution on the top is the output of reconstructor $ F_{\theta}(G_{\gamma}(z)) $ mapped from the generated distribution $ G_{\gamma}(z) $ \cite{veegan}. \\
 
 In Figure \ref{VEEGANgraph1}, suppose $ F_{\theta} $ is trained to be the inverse of $ G_{\gamma} .$ Because $ z $ is drawn from a Gaussian distribution, the outputs of reconstructor $ F_{\theta}(G_{\gamma}(z)) $ will also form a Gaussian distribution. However, $ G_{\gamma}(z) $ does not address all modes in $ p(x) $. As a result, $ F_{\theta} (x)$ will differ from $ F_{\theta}(G_{\gamma}(z)) $  and will likely be non-Gaussian. We detect mode collapse from the discrepancy between $ F_{\theta}(x) $ and $ F_{\theta}(G_{\gamma}(z)) $. In Figure \ref{VEEGANgraph2}, $ F_{\theta} $ is trained to map samples drawn from $ p(x) $ to outputs that form a Gaussian distribution. In this case, if $ G_{\gamma} $ suffers from mode collapse, $ F_{\theta}(G_{\gamma}(z)) $ will have a different distribution from $ F_{\theta}(x) $. Such difference will provide a strong learning signal to $ \gamma $ and $ \theta $ when the neural network performs update on parameters using stochastic gradient descent \cite{veegan}. \\
 
 We want to $ F_{\theta} $ to both act as the inverse of $ G_{\gamma} $ and map the target data distribution to a Gaussian. Mathematically, we can represent the difference between $ z $ and $ F_{\theta}(G_{\gamma}(z)) $ by the expected autoencoder $ l_{2} $ loss $ \mathbb{E}(\lVert z-F_{\theta}(G_{\gamma}(z))\rVert_{2}^{2}).$ We can represent how good $ F_{\theta} $ maps the true data distribution to a Gaussian by cross entropy $ H(z, F_{\theta}(x)) $. The probability distribution of noise vector $ z $ is Gaussian. Cross entropy measures how different the probability distribution of $ F_{\theta}(x) $ is from the probability distribution $ z $. If the two probability distributions are identical, then cross entropy $ H(z, F_{\theta}(x)) $ is at its minimum and equals to entropy of $ z $. For the reconstructor to both act as the inverse of $ G_{\gamma} $ and map the target data distribution to a Gaussian, we want to minimize the objective function 
 $ O_\mathrm{entropy}(\gamma, \theta) = \mathbb{E}(\lVert z-F_{\theta}(G_{\gamma}(z))\rVert_{2}^{2}) + H(z, F_{\theta}(x)).$ The cross entropy is written as $ H(z,F_{\theta}(x))=-\int p_{0}(z)\log\int p(x)p_{\theta}(\hat{z}|x) dx dz $. When $ z $ and $ x $ are in high dimensions, the integral is difficult to compute. Thus, we need to approximate $ O_\mathrm{entropy}(\gamma,\theta) $ with a more computable method.\\
 
 \subsection{Approximation of Objective Function}
 Let $ p_{\theta}(\hat{z}|x) $ be the distribution of reconstructor output given that the input is $ x $. Applied to $ x \sim p(x),$ let the output distribution of the reconstructor be \[p_{\theta}(\hat{z})=\int p_{\theta}(\hat{z}|x)p(x) dx.\] 
  
 \[H(z,F_{\theta}(x))=-\int p_{0}(z)\log p_{\theta}(\hat{z}) dz = -\int p_{0}(z)\log\int p(x)p_{\theta}(\hat{z}|x) dx dz.\]
  
 However, the integral on the right hand side is difficult to compute in closed form. Instead, we find an upper bound for $ H(z,F_{\theta}(x))$ by introducing a new term into the integral: $ q_{\gamma}(x|z)$. Denote $ q_{\gamma}(x|z) $ as the output of generator $ G_{\gamma} $. Using Jensen's inequality, adding and manipulating trivial terms in the integral, the authors of VEEGAN are able to find a computable KL divergence as upper bound to cross entropy $ H(z, F_{\theta}(x)) $. \\
 
 \begin{prop}[This is Equation 4 in \cite{veegan}]
 	\label{KL uppper bound}
 	Let $ q_{\gamma}(x|z)$ be the output of $ G_{\gamma} $. Then
 	\[H(z,F_{\theta}(x))=-\int p_{0}(z)\log p_{\theta}(\hat{z}) dz \leq \KL[q_{\gamma}(x|z)p_{0}(z) \parallel p_{\theta}(\hat{z}|x)p(x)] - \mathbb{E}[\log p_{0}(z)].\] 
 	
 \end{prop}

\begin{proof}(This proof is from Appendix A in \cite{veegan})
	\[-\log p_{\theta}(\hat{z}) = -\log \int p_{\theta}(\hat{z}|x)p(x)\frac{q_{\gamma}(x|z)}{q_{\gamma}(x|z)}dx.\]
	By Jensen's inequality, let the probability distribution be $ q_{\gamma}(x|z) $ and the random variable be $ \frac{p_{\theta}(\hat{z}|x)p(x)}{q_{\gamma}(x|z)} .$
	Because $ -\log(\frac{p_{\theta}(\hat{z}|x)p(x)}{q_{\gamma}(x|z)}) $ is strictly convex,
	\[ -\log \int p_{\theta}(\hat{z}|x)p(x)\frac{q_{\gamma}(x|z)}{q_{\gamma}(x|z)}dx\leq - \int q_{\gamma}(x|z) \log \frac{p_{\theta}(\hat{z}|x)p(x)}{q_{\gamma}(x|z)}dx\]
	
	Therefore, 
	\begin{align*}
		-\int p_{0}(z)\log p_{\theta}(\hat{z})dz &=  -\int p_{0}(z) \log \int p_{\theta}(\hat{z}|x)p(x)\frac{q_{\gamma}(x|z)}{q_{\gamma}(x|z)}dxdz \\
		&\leq -\int\int p_{0}(z)q_{\gamma}(x|z) \log \frac{p_{\theta}(\hat{z}|x)p(x)}{q_{\gamma}(x|z)}dxdz 
	\end{align*}
	 
	We add a trivial KL divergence to the right hand side of the inequality,
	\begin{align*}
		 -\int p_{0}(z)\log p_{\theta}(\hat{z})dz &\leq -\int\int p_{0}(z)q_{\gamma}(x|z) \log \frac{p_{\theta}(\hat{z}|x)p(x)}{q_{\gamma}(x|z)}dxdz \\
		&= \int\int p_{0}(z)q_{\gamma}(x|z) \log \frac{q_{\gamma}(x|z)}{p_{\theta}(\hat{z}|x)p(x)}dxdz + \int p_{0}(z)\log \frac{p_{0}(z)}{p_{0}(z)}dz. 
	\end{align*}
	
	For the upper term in the KL divergence on the right hand side, we have
	\[ \int p_{0}(z)\log p_{0}(z)dz = \int p_{0}(z)\log p_{0}(z) (\int q_{\gamma}(x|z) dx)dz = \int\int p_{0}(z)q_{\gamma}(x|z)\log p_{0}(z) dxdz. \]
	$ \int q_{\gamma}(x|z) dx = 1 $ because $ z $ is constant in this integral, and integrating with respect to $ x $ the probability density function of $ x $ given $ z $ equals 1. \\
	
	Thus,
	\begin{align*}
		H(z,F_{\theta}(x)) &\leq \int\int p_{0}(z)q_{\gamma}(x|z) \log \frac{q_{\gamma}(x|z)}{p_{\theta}(\hat{z}|x)p(x)}dxdz \\
		&\phantom{space} + \int\int p_{0}(z)q_{\gamma}(x|z)\log p_{0}(z) dxdz - \int p_{0}(z)\log p_{0}(z) dz \\
		&= \int\int  p_{0}(z)q_{\gamma}(x|z) \log \frac{q_{\gamma}(x|z)p_{0}(z)}{p_{\theta}(\hat{z}|x)p(x)} dxdz- \int p_{0}(z)\log p_{0}(z) dz \\
		&= \KL[q_{\gamma}(x|z)p_{0}(z) \parallel p_{\theta}(\hat{z}|x)p(x)] - \mathbb{E}[\log p_{0}(z)]
	\end{align*}
	
\end{proof}

By Proposition \ref{KL uppper bound}, the objective function $ O_\text{entropy}(\gamma, \theta) $ has upper bound
\[ O(\gamma, \theta) = \KL [q_{\gamma}(x|z)p_{0}(z) \parallel p_{\theta}(\hat{z}|x)p(x)] - \mathbb{E}[\log p_{0}(z)] + \mathbb{E}[\lVert z - F_{\theta}(G_{\gamma}(z)) \rVert_{2}^{2}].\] 

$ O(\gamma, \theta) $ is the upper bound of the reconstructor loss function. To minimize the reconstructor loss function, we want to minimize its upper bound. $ O(\gamma, \theta) $ is minimized when the KL divergence and autoencoder loss both equal to 0, which means $ q_{\gamma}(x|z)p_{0}(z) = p_{\theta}(\hat{z}|x)p(x) $ and $ z = F_{\theta}(G_{\gamma}(z)). $ Theorem \ref{veeganmin} states the minimum of the reconstructor upper bound $ O(\gamma, \theta) $. It also states the distribution of reconstructor output $ p_{\theta}(\hat{z}) $ and generator output $ q_{\gamma}(x) $ when $ O(\gamma, \theta) $ is minimized. 
\\

\begin{thm}[This is Proposition 1 in \cite{veegan}]
	\label{veeganmin}
	If $ \gamma^{*} $ and $ \theta^{*} $ achieves $ \min (O(\gamma, \theta)), $ then $ O(\gamma^{*}, \theta^{*}) = H(p_{0}),$ with $ p_{0} $ as the target data distribution.\\
	Further, 
	$p_{\theta^{*}}(\hat{z}) = \int p_{\theta^{*}}(\hat{z}|x)p(x)dx = p_{0}(z),$ 
	and ${q_{\gamma^{*}}(x) = \int q_{\gamma^{*}}(x|z)p_{0}(z)dz = p(x)}.$
\end{thm}

\begin{proof}(This proof is from Appendix B in \cite{veegan})
	Since ${\KL [q_{\gamma}(x|z)p_{0}(z) \parallel p_{\theta}(\hat{z}|x)p(x)]} \geq 0 $ and $ {\mathbb{E}[\lVert z - F_{\theta}(G_{\gamma}(z)) \rVert_{2}^{2} \geq 0}, $ we can only minimize $ O(\gamma, \theta) $ by setting $ \gamma^{*} $ and  $ \theta^{*} $ such that ${q_{\gamma^{*}}(x|z)p_{0}(z) = p_{\theta^{*}}(\hat{z}|x)p(x)}  $ and $ {z = F_{\theta^{*}}(G_{\gamma^{*}}(z))}.$ Then, \[O(\gamma^{*}, \theta^{*}) =  - \mathbb{E}[\log p_{0}(z)] = - \int p_{0}(z)\log p_{0}(z) = H(p_{0}).\] 
	
	By the condition ${q_{\gamma^{*}}(x|z)p_{0}(z) = p_{\theta^{*}}(\hat{z}|x)p(x)} , $ if we integrate both sides by $ x ,$ we obtain 
	\[ p_{\theta^{*}}(\hat{z}) = \int p_{\theta^{*}}(\hat{z}|x) p(x) dx = \int q_{\gamma^{*}}(x|z)p_{0}(z)dx = p_{0}(z). \]
	If we integrate both sides by $ z, $ we obtain 
	\[q_{\gamma^{*}}(x) = \int q_{\gamma^{*}}(x|z)p_{0}(z)dz = \int p_{\theta^{*}}(\hat{z}|x) p(x)dz = p(x) \]	
\end{proof}

\section{Experiments}
In my experiments, I compare the severity of mode collapse in images of MNIST digits generated by PacGAN with packing degree of 2, VEEGAN, and DCGAN. In the literature, PacGAN \cite{pacgan} and VEEGAN \cite{veegan} explicitly stated that their architectures effectively reduce mode collapse, while DCGAN \cite{dcgan} is not designed to reduce mode collapse. I used MNIST digits as the training dataset for the GAN architectures. With 60,000 images of handwritten digits and 6,000 images for each digit, the frequency of appearance of each digit in the MNIST dataset is equal. The target probability distribution of digits is uniform. I conducted 10 trials on DCGAN, PacGAN, and VEEGAN respectively, 100 digits were generated by each GAN architecture in one trial. The digits generated by DCGAN and PacGAN are classified after 60 epochs of training. Because VEEGAN generates more blurred images compared to DCGAN and PacGAN, the digits generated by VEEGAN are classified after 400 epochs of training for better image quality. Figure \ref{digits} displays generated digit samples by DCGAN (left), PacGAN (middle), and VEEGAN (right). In my experiments, the DCGAN and PacGAN code are based on \cite{dcgan}, and I used  code from \cite{veegan} for the VEEGAN code. 

\begin{figure}[H]
	\centering
	\begin{subfigure}[b]{0.3\textwidth}
		\centering
		\includegraphics[width=0.9\linewidth]{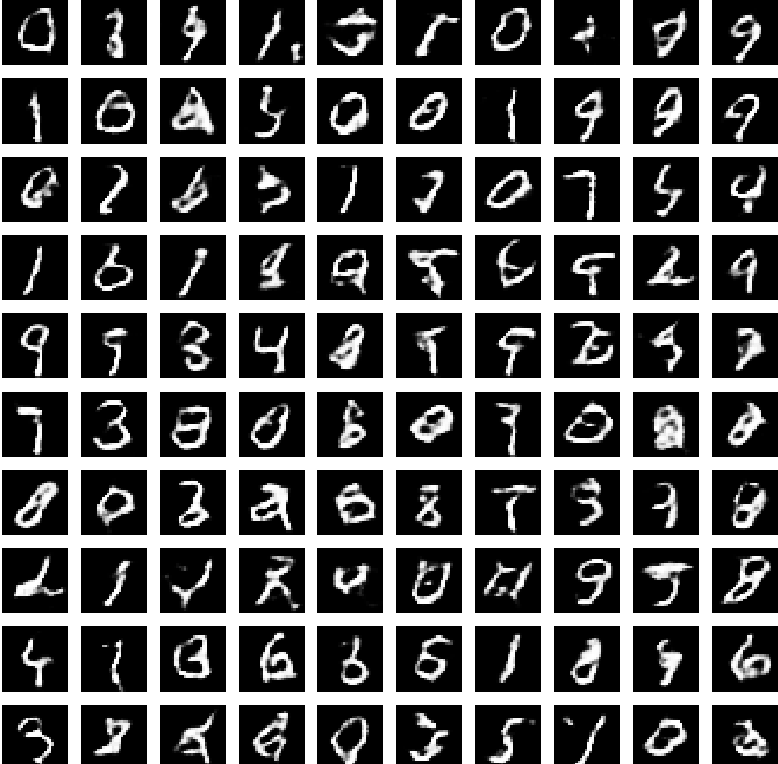}
	\end{subfigure}%
	\begin{subfigure}[b]{0.3\textwidth}
		\centering
		\includegraphics[width=0.9\linewidth]{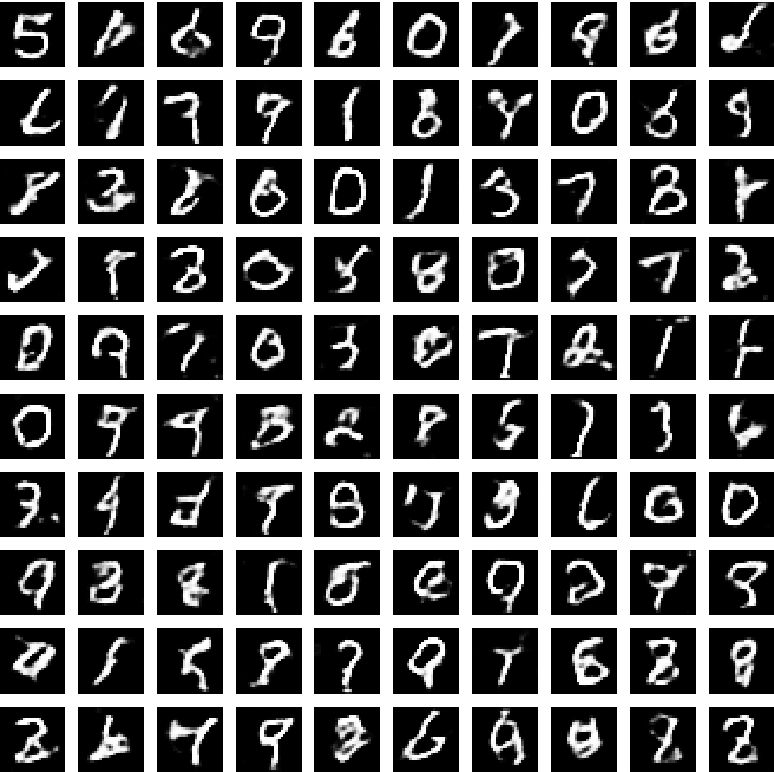}
	\end{subfigure}%
	\begin{subfigure}[b]{0.3\textwidth}
		\centering
		\includegraphics[width=0.9\linewidth]{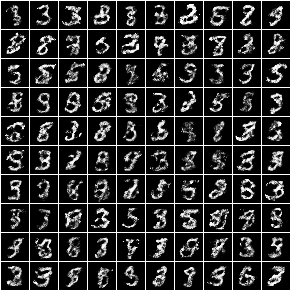}
	\end{subfigure}%
	\caption{The left panel displays 100 digit samples generated by DCGAN; the middle panel displays 100 digit samples generated by PacGAN; the right panel displays 100 digit samples generated by VEEGAN. }
	\label{digits}
\end{figure}

There are multiple explanations to why VEEGAN generates blurrier images than DCGAN and PacGAN. One explanation is the objective function in VEEGAN has a term $ \mathbb{E}[\lVert \hat{z} - F_{\theta}(G_{\gamma}(z)) \rVert_{2}^{2}] $. The term is the average l2-loss of reconstructor error. This term encourages VEEGAN to generate images that minimizes \textit{average} reconstructor error. Under a system that minimizes average error, the VEEGAN generator is encouraged to generate images whose pixel values are the average of all pixel values in the dataset, instead of generating sharper individual images. Another explanation is that the VEEGAN architecture heavily penalizes images with features in the wrong place, but is more lenient toward blurry features. In contrast, DCGAN and PacGAN heavily penalize blurry images because they look inauthentic compared to the real images with sharp features. An alternate explanation is that VEEGAN includes an autoencoder that compresses data into a latent space. Because the latent space is in a smaller dimension than the data space, it is hard to pass all the information in the image to the latent space. Meanwhile VEEGAN tries to minimizes the output loss. As a result, the images generated by VEEGAN capture fewer details than real images, and the generated image looks like each pixel is an average of the pixels of several real images. \\

From the distributions of images generated by DCGAN, PacGAN, and VEEGAN, I find that DCGAN and PacGAN generate an excessive number of 0, 1, and 7 compared to the MNIST training dataset. DCGAN displays severe mode collapse on digits 2 and 4. In some trials, there were only one or two images of 2 among the 100 images generated by DCGAN. PacGAN also displays mode collapse on digit 2 and 4, but to a slightly lesser extent. From my experiment results, DCGAN and PacGAN have a proclivity to generate digits with simple features and avoid generating digits with sophisticated features. In contrast, VEEGAN generates an excessive number of 3, 5, and 8, but it displays severe mode collapse on number 1 and 6. In one of the trials, number 6 was not present at all among the 100 images. In two other trials, number 1 was not present at all. VEEGAN favors generating digits with curves and sophisticated features, and avoids generating digits with straight lines and simple features. Figure \ref{digitfreq} shows the frequency of digits generated in one trial by DCGAN, PacGAN, and VEEGAN. I choose not to display the average frequency of digits because the digit distribution in each trial is different. Taking the average of 10 distributions will reduce the effect of mode collapse illustrated in individual trials.\\

\begin{figure}[h]
	\centering
	\includegraphics{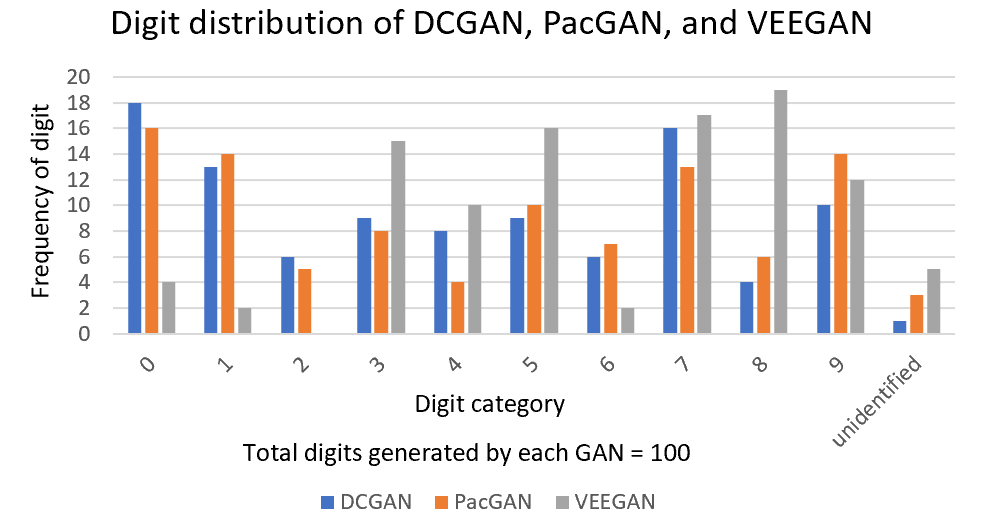}
	\caption{This chart displays the frequency of each digit generated by DCGAN, PacGAN, and VEEGAN, from left to right. DCGAN and PacGAN produce similar digit distributions, while VEEGAN produces a substantially different digit distribution from the former two.}
	\label{digitfreq}
\end{figure}

I use Kullback-Leibler divergence (KL divergence) as the metric to measure the severity of mode collapse between the distribution of generated digits and MNIST digits. KL divergence is a measure of how one probability distribution is different from the reference distribution \cite{klbook}\cite{kldivergence}. For discrete probability distributions $ P $ and $ Q $ defined on the same probability space $ \mathcal{X} $, consider $ P $ as the reference distribution. Then for random variable $ X\in\mathcal{X} $, the KL divergence between $ Q $ and $ P $ is defined as $ \KL(Q||P) = \displaystyle\sum_{X\in\mathcal{X}} Q(X) \log \frac{Q(X)}{P(X)} $. When $ Q(X) = 0 $, the contribution of the corresponding term is 0, because $ \displaystyle\lim_{x\to 0^{+}} x\log x = 0 $. But when for some $ X $, $ P(X)=0 $ and $ Q(X) \neq 0 $, $ \KL(Q||P) = \infty $. However, in experiments, it is very likely that some digits are not generated, and the digit category will have probability of 0. To make the KL divergence computable when the probability of a digit category is 0, I approximate the probability of the digit as $ 10^{-10} $ instead of 0. In every trial of my experiment, I computed the KL divergence between generated digit distribution from each GAN architecture and digit distribution from MNIST dataset. The digit distribution from MNIST dataset is the reference distribution. The KL divergence between generated distribution by each GAN and  the MNIST distribution is averaged over 10 trials. A low KL divergence indicates the generated digit distribution resembles the distribution of MNIST dataset and has minimal mode collapse. A high KL divergence implies the generated digit distribution deviates from the distribution of MNIST dataset and suffers severe mode collapse. The average KL divergence over 10 trials between generated digit distribution by different GANs and MNIST dataset is displayed in Table \ref{kltable}.

\begin{table}[H]
	\centering
	\begin{tabular}{|c|c|}
		\hline
		& Average KL divergence \\
		\hline
		DCGAN & 0.581\\
		\hline
		PacGAN & 0.522\\
		\hline
		VEEGAN & 0.876\\
		\hline
	\end{tabular}
	\caption{This table displays the average KL divergence between the generated digit distributions from three different GAN architectures and the MNIST digit distribution. Each GAN architecture generated 100 images of digits each trial for 10 trials. In each trial, the KL divergence between generated and MNIST digit distribution is computed. The average KL divergence over 10 trials is displayed in this table.}
	\label{kltable}
\end{table}

Contrary to the experiment results presented by \cite{pacgan} and \cite{veegan}, my experiments show that PacGAN and DCGAN display mode collapse to a similar extent. With only 10 trials, I am unable to determine if a 0.059 difference in KL divergence is statistically significant. VEEGAN demonstrates more severe mode collapse than the former two. My experiment result is incongruent with the experiment results in \cite{pacgan} and \cite{veegan}.

\section{Conclusion}
Both PacGAN and VEEGAN intend to reduce mode collapse between generated and target distributions. However, the approaches they take to prove their theoretical results are very different in spirit. The theoretical result of PacGAN involves a geometric proof drawing theorems from information theory, linear algebra, and topology; while the theoretical result of VEEGAN is more algebraic and comprises of approximation and minimization of an objective function. In my experiments, the result does not show that PacGAN demonstrates significantly less mode collapse than DCGAN, contrary to the experiment results of the PacGAN authors. VEEGAN demonstrates more severe mode collapse and poorer image quality than both DCGAN and PacGAN. 

\section{Acknowledgment}
I would like to thank Professor Peter Ramadge for his guidance and keen insight on my work. He points me to possible new directions when I am lost. He never hesitates to help me with my questions, no matter they pertain to mathematics, programming, or even formatting. He has taught me that no problem is too small. I would like to also thank Jonathan Spencer, Vineet Bansal, Hossein Valavi and Yinqi Tang for taking their personal time to help me with debugging my programs. Your help facilitated my progress in running my experiments that would otherwise not finish in time. It means a lot to me.

\newpage
\bibliography{biblio}
\bibliographystyle{plain}

\end{document}